\documentclass{article}

\usepackage{microtype}
\usepackage{graphicx}
\usepackage{subfigure}
\usepackage{booktabs} 

\usepackage{hyperref}
\usepackage{soul} 

\usepackage{amsthm}
\usepackage{tikz} 
\usetikzlibrary{arrows.meta} 
\usetikzlibrary{decorations.markings}
\usetikzlibrary{backgrounds}

\input amssym.def
\input amssym.tex

\theoremstyle{plain}
\newtheorem{theorem}{Theorem}[section] 
\newtheorem{lemma}{Lemma}[section]
\newtheorem{remark}{Remark}[section]
\newtheorem{corollary}{Corollary}[section]

\theoremstyle{definition}
\newtheorem{defn}{Definition}[section] 
\newtheorem{exmp}{Example}[section] 

\usepackage[accepted]{icml2019}

\begin{document}

\twocolumn[
\icmltitle{The Complexity of Morality: Checking Markov Blanket Consistency with DAGs via Morality}




\begin{icmlauthorlist}
\icmlauthor{Yang Li}{monash}
\icmlauthor{Kevin Korb}{monash}
\icmlauthor{Lloyd Allison}{monash}
\end{icmlauthorlist}

\icmlaffiliation{monash}{Faculty of Information Technology, Monash University, Clayton, Australia}

\icmlcorrespondingauthor{Yang Li}{yang.kelvinli@monash.edu}


\vskip 0.1in
]



\printAffiliationsAndNotice{}  

\begin{abstract}
A family of Markov blankets in a faithful Bayesian network satisfies the symmetry and consistency properties. In this paper, we draw a bijection between families of consistent Markov blankets and moral graphs. We define the new concepts of weak recursive simpliciality and perfect elimination kits. We prove that they are equivalent to graph morality. In addition, we prove that morality can be decided in polynomial time for graphs with maximum degree less than $5$, but the problem is NP-complete for graphs with higher maximum degrees. 
\end{abstract}

\section{Introduction}
Introduced by \citet{pearl1988probabilistic} as the smallest subset of
variables in a Bayesian network, given which the target variable is
conditionally independent from the rest of the variables, the Markov blanket\footnote{Originally, this is how \citet{pearl1988probabilistic}
  defined ``Markov boundaries'', but the literature has migrated
  ``Markov blankets'' to this minimalist sense.}
  has became popular for feature selection \cite{koller1996toward} and scaling up learning causal models \cite{ramsey2016million}. For a comprehensive review of Markov blanket discovery and its applications in structure learning, we refer the readers to \cite{aliferis2010localb}.  
In a faithful Bayesian network, the \textit{Markov blanket} of a target variable consists of its parents, children and children's other parents (a.k.a., spouses) (Figure \ref{fg:mb_example}).  
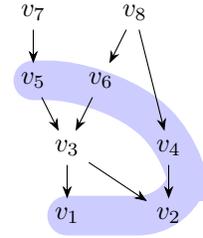
\begin{figure}
\centering
\begin{tikzpicture}[scale=0.9]
\begin{scope}[>={Stealth[black]},              
              every edge/.style={draw=black}]
    \node (A) at (0.5,0) {$v_1$};
    \node (B) at (2,0) {$v_2$};
    \node (C) at (0.5,1) {$v_3$};
    \node (D) at (2,1) {$v_4$};
    \node (E) at (0,2) {$v_5$};
    \node (F) at (1,2) {$v_6$};
    \node (G) at (0,3) {$v_7$};    
    \node (H) at (1.5,3) {$v_8$};   
    \path [->] (G) edge (E);
    \path [->] (H) edge (F);
    \path [->] (H) edge (D);
    \path [->] (E) edge (C);
    \path [->] (F) edge (C);
    \path [->] (C) edge (A);
    \path [->] (C) edge (B);
    \path [->] (D) edge (B);
\end{scope}
\begin{pgfonlayer}{background}
	\draw[rounded corners=2em,line width=1.5em,blue!20,cap=round]
		(A.center) -- (B.east) -- (D.east) -- (F.center) -- (E.center);
\end{pgfonlayer}
\end{tikzpicture}
\caption{The Markov blanket of $v_3$ in this faithful Bayesian network is $\{v_5,v_6,v_1,v_2,v_4\}$.}
\label{fg:mb_example}
\end{figure}
A set $B(V)=\{B(v_1), \dots, B(v_n)\}$ of subsets of variables is considered to be a valid family of Markov blankets for the variables $V=\{v_1, \dots,v_i\}$ in a faithful Bayesian network if it satisfies the symmetry and consistency properties. The \textit{symmetry property}, which states $v_i \in B(v_j)$ if and only if $v_j \in B(v_i)$ is a consequence of the graphical interpretation of Markov blankets in faithful Bayesian networks. The \textit{consistency property} guarantees that there exists at least one directed acyclic graph (DAG) s.t. the Markov blanket of $v_i$ in it equals $B(v_i)$ for all $v_i \in V$. 

Until recently, there have been few literature paying attention to
Markov blankets consistency. A learned family of Markov blankets, if
not read off from a DAG, often does not tell explicit relations among
variables. This does not stop symmetry being quickly checked and
enforced (in various of ways), but makes it non-trivial to check
consistency. Without being consistent with a DAG, these Markov
blankets could lead to contradictory local structures, which have to
be resolved in applying local to global structure learning, which is
our underlying motivation.

In this paper, we relate graph morality to Markov blanket consistency,
and present polynomial time algorithms for checking morality for
undirected graphs with various of maximum degrees. In Section
\ref{sec:pre} we develop the important concepts for this paper. In
Section \ref{sec:wrs}, we prove the equivalence of certain properties
to morality. In Section \ref{sec:complexity}, we analyse the
computational complexity of checking morality for graphs with various
maximum degree.


\section{Preliminary} 
\label{sec:pre}
Throughout this paper, we consider only connected graphs. For
simplicity, we refer to them as \textit{graphs}, which is a pair
$G = (V, E)$ comprising a set $V$ of vertices (or nodes) together with
a set $E$ of edges (or arcs) connecting pairs in $V$. If $E$ is a set
of ordered pairs of distinct vertices in $V$, then $G$ is a
\textit{directed graph}. For vertices $u,v \in V$, we use $d(u)$ to denote the degree of $u$, $\Delta(G)$ to denote the maximum degree of $G$, $uv$ to
represent an (undirected) edge and $\overrightarrow{uv}$ to represent
a directed edge from $u$ to $v$. A \textit{hybrid graph} is a graph
consisting of both directed and undirected edges. The
\textit{skeleton} of a hybrid graph is the undirected graph obtained
by dropping directions of all directed edges. A directed graph is
called a \textit{directed acyclic graph} if it contains no directed
cycles. In a DAG $G=(V,E)$, $u$ is a \textit{parent} of $v$, denoted
by $u \in P_G(v)$ (or $v$ is a \textit{child} of $u$) if there is a
directed edge $\overrightarrow{uv} \in
E$. 

Let $\mathcal{P}$ be a joint probability distribution of the random variables in $V$ and $G=(V,E)$ be a DAG. We say the two together form a \textit{Bayesian network} $<G, \mathcal{P}>$ if it satisfies the Markov condition.
\begin{defn}
\label{def:mb}
Let $<G=(V,E),\mathcal{P}>$ be a Bayesian network. The \textbf{Markov
  blanket} of $u \in V$ in the Bayesian network, denoted by $B(u)$, is
the minimum subset of variables satisfying
$u \!\perp\!\!\!\perp_{\mathcal{P}} v \mid B(u)$ for each
$v \in V\setminus B[u]$, where $B[u]=B(u)\cup \{u\}$.
\end{defn}

\begin{defn}
\label{def:moral_g}
The \textbf{moral graph} of a directed acyclic graph $G=(V,E)$ is the skeleton of the hybrid graph $H=(V,E\cup F)$, where $F=\{uv \mid u,v \in P_G(x), \{\overrightarrow{uv}, \overrightarrow{vu}\}\cap E = \emptyset, \forall x\in V\}$. 
\end{defn}
The above definition implicitly states a trivial \textit{moralization} process that turns a DAG into a moral graph. That is, by joining all pairs of non-adjacent parents in the DAG, then dropping all the directions. We call $F$ the set of \textit{filled-edges}. 

\begin{exmp}
Figure \ref{fg:envelope} shows a DAG and its moral graph that is obtained by joining $v_3$ and $v_4$ then dropping all the directions in the hybrid graph. 
\label{ex:moral_graph}
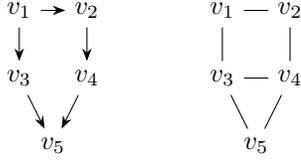
\begin{figure}
\centering
\begin{tikzpicture}[scale=0.9]
\begin{scope}
    \node (A) at (0,0) {$v_1$};
    \node (B) at (1,0) {$v_2$};
    \node (C) at (0,-1) {$v_3$};
    \node (D) at (1,-1) {$v_4$};
    \node (E) at (0.5,-2) {$v_5$};
    
    \node (F) at (3,0) {$v_1$};
    \node (G) at (4,0) {$v_2$};
    \node (H) at (3,-1) {$v_3$};
    \node (I) at (4,-1) {$v_4$};
    \node (J) at (3.5,-2) {$v_5$};
\end{scope}

\begin{scope}[>={Stealth[black]},
              every node/.style={fill=white,circle},
              every edge/.style={draw=black}]
    \path [->] (A) edge (B);
    \path [->] (A) edge (C);
    \path [->] (B) edge (D);
    \path [->] (C) edge (E);
    \path [->] (D) edge (E);
    
    \path [-] (F) edge (G);
    \path [-] (F) edge (H);
    \path [-] (G) edge (I);
    \path [-] (H) edge (I);
    \path [-] (H) edge (J);
    \path [-] (I) edge (J);
\end{scope}
\end{tikzpicture}
\caption{A DAG $G$ (left) and its moral graph $H$ (right), in which $v_3v_4$ is a filled-edge.}
\label{fg:envelope}
\end{figure}
\end{exmp}

For any spouse $v$ of $u$ that is neither a parent nor child of $u$, the two vertices $u,v$ must be connected in order to produce the moral graph of $G$. Hence, for each vertex $u \in V$, its Markov blanket in the DAG is identical to its neighbours in the moral graph. For example, in Figure \ref{fg:envelope} $B_G(v_3)=\{v_1,v_5,v_4\}=N_H(v_3)$.

For simplicity, if $V' \subset V$ then we use $G-V'$ to denote the \textit{induced subgraph} $G[V\setminus V']$ over the nodes in $V\setminus V'$. If $V'=\{u\}$, then we use $G-u$. If $V'=V(H)$, then we use $G-H$. Similarly, if $E' \subset E$ then we use $G+E'$ and $G-E'$ to denote $(V,E \cup E')$  and $(V,E\setminus E')$ respectively. If $E' = \{uv\}$ then we use $G+uv$ or $G-uv$ instead.

It is also useful to define the \textit{closed neighbours} of $u$ in $G$ as $N_G[u]=N_G(u)\cup \{u\}$ and the neighbours of a subgraph $H \subset G$ as $N_G(H)=\{u \in V\setminus V(H) \mid uv \in E, \forall v \in V(H)\}$.

\begin{defn}
A \textbf{simplicial node} in a graph is a node whose neighbours form a complete subgraph (a.k.a., clique). 
\end{defn}

\begin{defn}
Let $G=(V,E)$ be a graph. The \textbf{deficiency} of a node $x$ in $G$ is $D(x)=\{uv \notin E \mid u, v \in N(x)\}$.
\end{defn}
A node $u$ is simplicial in $G$ if and only if $D(u)=\emptyset$. That
is, no edge needs to be filled in to make the neighbours of $u$ a
clique. For all $u \in V$ if $D_G(u)\neq \emptyset$, then we write $D(G)\neq \emptyset$. If $\exists u \in V$ s.t. $D(u)=\emptyset$, then we write $D(G)=\emptyset$ .
\begin{exmp}
In the moral graph $H$ as shown in Figure \ref{fg:envelope}, $D_H(v_1)=\{v_2v_3\}$ and $D_H(v_5)=\emptyset$. 
\end{exmp}

A chordal graph $G=(V,E)$ is also known to be \textit{recursively
  simplicial}. That is, there exists a simplicial node $x$ s.t. the
induced subgraph $G-x$ is also recursively simplicial. Next, we
introduce a similar concept, but which requires indefinite edge removal in
addition to deleting a simplicial node.

\begin{defn}
\label{def:wrs}
A graph $G=(V,E)$ is \textbf{weakly recursively simplicial} if
$\exists x \in V$ with $D_G(x)=\emptyset$ and
$\exists E'\subseteq E(G[N(x)])$ s.t. the subgraph $G'=G-x-E'$ is
weakly recursively simplicial.
\end{defn}

\begin{exmp}
\begin{figure}
\centering
\begin{tikzpicture}[scale=0.9]
\begin{scope}
    \node (A) at (0,0) {$v_1$};
    \node (B) at (1,0) {$v_2$};
    \node (C) at (0,-1) {$v_3$};
    \node (D) at (1,-1) {$v_4$};
    \node (E) at (0.5,-2) {$v_5$};
    
    \node (F) at (3,0) {$v_1$};
    \node (G) at (4,0) {$v_2$};
    \node (H) at (3,-1) {$v_3$};
    \node (I) at (4,-1) {$v_4$};
    \node (J) at (3.5,-2) {$v_5$};
\end{scope}

\begin{scope}[>={Stealth[black]},
              every node/.style={fill=white,circle},
              every edge/.style={draw=black}]
    \path [-] (A) edge (B);
    \path [-] (A) edge (C);
    \path [-] (B) edge (D);
    \path [-] (C) edge (D);
    \path [-] (C) edge (E);
    
    \path [-] (F) edge (G);
    \path [-] (F) edge (H);
    \path [-] (G) edge (I);
    \path [-] (H) edge (I);
    \path [-] (H) edge (J);
    \path [-] (I) edge (J);
\end{scope}
\end{tikzpicture}
\caption{An example of a non-weakly recursively simplicial graph $G$ (left) and a weakly recursively simplicial graph $H$ (right).}
\label{fg:wrs}
\end{figure}
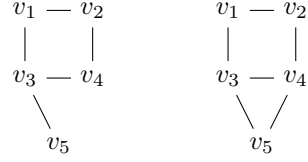

In Figure \ref{fg:wrs}, $H$ is a weakly recursively simplicial (WRS)
graph, because it can be turned into the empty graph by recursively
eliminating $\{v_5,v_3v_4\}$, $\{v_3\}$, $\{v_4\}$, $\{v_1\}$,
$\{v_2\}$, where each node is simplicial in the sequence of
subgraphs. The graph $G$, however, is not WRS because there is no such
sequence.
\end{exmp}

If a graph is recursively simplicial (i.e., chordal), it is also weakly recursively simplicial with $E'=\emptyset$ at each recursive step. The converse, however, is not true. For example, the graph $H$ in Figure \ref{fg:wrs} is WRS but not chordal. To further explore this recursive definition, we introduce the following concepts.

\begin{defn}
An \textbf{ordering} of a graph $G=(V,E)$ with $n$ vertices is a bijection $\alpha: \{1, \dots, n\} \leftrightarrow V$. 
\end{defn}
For simplicity, we use $\alpha=\{v_1,\dots,v_n\}$ to denote the ordering $\alpha$ s.t. $\alpha(i)=v_i$ for $i \in [1,n]$.

\begin{defn}
A set of \textbf{excesses} of a graph $G=(V,E)$ w.r.t. an ordering $\alpha$ is a bijection $\epsilon_{\alpha}: \{\alpha(1),\dots,\alpha(n)\} \leftrightarrow \{\epsilon_{\alpha}(\alpha(1)), \dots, \epsilon_{\alpha}(\alpha(n))\}$, where each $\epsilon_{\alpha}(\alpha(i)) \subseteq E(G[N(\alpha(i))])$ consists of some edges between the neighbours of $\alpha(i)$.
\end{defn}
The composition $\kappa=(\alpha,\epsilon_{\alpha})$ of an ordering and a set of excesses (w.r.t. $\alpha$) is called an \textit{elimination kit} of a graph $G$. We use the convention $\kappa(0)=\emptyset$ and let $\kappa(i)=\{\alpha(i), \epsilon_{\alpha}(\alpha(i))\}$ be the $i^{th}$ elimination kit. Hence, we can define the \textbf{subgraph, not yet elimination graph} \textit{eliminated graph} $G^i=G-\{\kappa(0),\dots,\kappa(i)\}$ for $i \in [0,n]$. 
\begin{exmp}
\label{ex:ek}
An ordering $\alpha=\{v_5,v_3,v_4,v_1,v_2\}$ and a set of excesses $\epsilon_{\alpha}=\{\emptyset,\emptyset,\emptyset,\emptyset,\emptyset\}$ form an elimination kit of $H$ in Figure \ref{fg:wrs}. 
\end{exmp}

\begin{defn}
Let $G=(V,E)$ be a graph and $\kappa=(\alpha, \epsilon_{\alpha})$ be an elimination kit of $G$. Then $\kappa$ is a \textbf{perfect elimination kit} (PEK) of $G$ if each node $x \in V$ satisfies $D_{G^{\alpha^{-1}(x)-1}}(x)=\emptyset$.
\end{defn}

\begin{exmp}
\label{ex:pek}
The elimination kit in Example \ref{ex:ek} is not perfect, because $D_{H^1}(v_3)\neq \emptyset$. The only PEK for $H$ is when $\alpha=\{v_5,v_3,v_4,v_1,v_2\}$ and $\epsilon_{\alpha}=\{\{v_3v_4\},\emptyset,\emptyset,\emptyset,\emptyset\}$. 
\end{exmp}
Not all graphs have a PEK and some have more than one. In the next section, we prove that having a PEK is equivalent to being moral. According to the PEK in Example \ref{ex:pek}, the node $v_3$ is simplicial in the eliminated graph $H^1$ but not in $H$, so we say $v_3$ is a \textit{locally simplicial} node. Similarly, $v_4,v_1$ and $v_2$ are also locally simplicial. 

\begin{defn}
Let $G=(V,E)$ be a graph and $\kappa=(\alpha, \epsilon_{\alpha})$ be
an elimination kit of $G$. It is a \textbf{partial perfect elimination
  kit}  if there exists a non-empty eliminated graph $G^i \subset G$ s.t. $D(G^i)\neq \emptyset$ and $D_{G^{j-1}}(\alpha(j))=\emptyset$ for $j \in [1,i]$.
\end{defn}
A 4-cycle has no partial PEK, because it has no simplicial node. A graph that has a PEK may also have a partial PEK.
\begin{exmp}
Example \ref{ex:ek} is a partial PEK, because $D(G^1)\neq \emptyset$ and $D_{G^0}(v_5)=\emptyset$. 
\end{exmp}

\section{Morality, weak recursive simpliciality and perfect elimination kits}
\label{sec:wrs}
In this section, we prove the equivalence of some properties to morality. We first show that there is a one-to-one correspondence between being WRS and having a PEK. 

\begin{theorem}
\label{thm:wrs_has_pek}
A graph is weakly recursively simplicial if and only if it has a perfect elimination kit. 
\end{theorem}
\begin{proof}
 If $G=(V,E)$ is WRS, the simplicial node $x$ and the edges
$E' \subset E(G[N_G(x)])$ removed at each step of the recursion form
an ordering and a set of excesses, because the $x$ at each step of
the recursion is locally simplicial. Hence, $G$ has a PEK. The
converse is also true because if $G$ has a PEK, it can be eliminated
recursively by following the PEK to get to the empty graph.
\end{proof}

Next, we show the equivalence between moral graphs and WRS graphs. This is proved by the following two lemmas. 
\begin{lemma}
\label{lm:moral_implies_wrs}
Let $G=(V,E)$ be a DAG and $H$ be the moral graph of $G$. Then $H$ is weakly recursively simplicial. 
\end{lemma}

\begin{proof}
The lemma is proved by induction on the number of nodes. Let $G(n)$
and $H(n)$ denote, respectively, a DAG and its moral graph over a set of $n$ nodes. The lemma is true for $n \le 3$, because all graphs containing three nodes or less are WRS. Assuming $H(n)$ is WRS for $n \ge 3$. We want to show that the moral graph $H(n+1)$ of DAG $G(n+1)$ is also WRS. Each DAG contains a sink and it becomes simplicial in the DAG's moral graph, because its parents form a clique after moralization. Hence, $H(n+1)$ contains a simplicial node $x$. By removing $x$ from the DAG we obtain a subgraph $G(n)$ that is also a DAG and its moral graph $H(n) \subset H(n+1)$. The inductive hypothesis assumes that each moral graph $H(n)$ is WRS. Hence, $H(n+1)$ is also WRS. 
\end{proof}

\begin{lemma}
\label{lm:wrs_implies_moral}
Let $H=(V,E)$ be a weakly recursively simplicial graph. Then $H$ is the moral graph of a DAG. 
\end{lemma}

\begin{proof}
  The lemma is proved by induction on the number of nodes $n$. The
  statement is true for $n=1$, because a single node graph $H(1)$ is
  both the moral graph of $G(1)$ and a WRS graph. Assume
$H(n)$ with $n\ge 1$ is WRS hence the moral graph of a DAG $G(n)$, we want to show that a WRS graph $H(n+1)$ is the moral graph of a DAG $G(n+1)$. By definition, $H(n+1)$ has a simplicial node $x$ and its excess $\epsilon(x)$ s.t. $H(n+1)-x-\epsilon(x)$ is WRS. By the inductive assumption, $H(n)$ is the moral graph of a DAG $G(n)$. Hence, by adding $x$ to $G(n)$ as a sink, we obtain a DAG $G(n+1)$, whose moral graph is $H(n+1)$.
\end{proof}

\begin{theorem}
\label{thm:wrs_equal_moral}
A graph is weakly recursively simplicial if and only if it is the moral graph of a DAG. 
\end{theorem}
\begin{proof}
The theorem follows from Lemma \ref{lm:moral_implies_wrs} and Lemma \ref{lm:wrs_implies_moral}. 
\end{proof}

The next lemma states that a moral graph can be eliminated by starting from any simplicial node. 
\begin{lemma}
\label{lm:wrs_start_any_sim}
If $H=(V,E)$ is moral and $x \in V$ is any vertex with $D_H(x)=\emptyset$, there is a perfect elimination kit $\kappa=(\alpha,\epsilon_{\alpha})$ with $\alpha(1)=x$. 
\end{lemma}
\begin{proof}
Let $G=(V,F)$ be a DAG, whose moral graph is $H$. For any sink $x$ in $G$, the subgraph $G'=G-x$ is also a DAG. Let $H'$ be the moral graph of $G'$, so $H'=H-x-f$ for $f \subset E(G[N_H(x)])$. By Theorem \ref{thm:wrs_has_pek} and Theorem \ref{thm:wrs_equal_moral}, $H'$ has a PEK $\kappa'=(\beta,\epsilon_{\beta})$. Hence, adding $x$ and $f$ to the front of $\beta$ and $\epsilon_{\beta}$ results in a PEK $\kappa=(\alpha,\epsilon_{\alpha})$ of $H$ s.t. $\alpha(1)=x$.
\end{proof}

\section{Complexity}
\label{sec:complexity}
\citet{verma1993deciding} proved that deciding morality for an
arbitrary graph is NP-complete. This is not only because the number of edges
between a simplicial node's neighbours is exponential in its degree,
but also because the deletion of some edges can stop a node being simplicial in
any following recusive step, which cannot be anticipated at
the time of deletion. In this section, we look at restricted
graphs. In particular, graphs with limited maximum degree. We develop
polynomial time algorithms for maximum degree $3$ and $4$
graphs. Furthermore, we prove that the NP-completeness still hold for
maximum degree $5$ graphs by modifying the reduction from 3-CNFs to
graphs as shown in \cite{verma1993deciding}.

It is trivial to check morality for graphs with maximum degree less
than or equal to 2. To prove our results for maximum degree $3$ and $4$
graphs, we prove the following lemmas first. Some of these lemmas are proved by contradiction. Given a graph $G$ is moral, the general strategy is to assume a subgraph of interest $G'=G-x-F$ is not moral, which is obtained by removing a simplicial node $x$ and some edges $F \subseteq E(G[N_G(x)])$ from $G$. And show that if the assumption is true, then $\forall F' \subseteq E(G[N_G(x)])$ s.t. $F' \neq F$ the subgraph $G''=G-x-F'$ is not moral. This contradicts to the premise that $G$ is moral, so the subgraph of interest $G'$ must be moral. By Lemma \ref{lm:wrs_start_any_sim}, $x$ can be any simplicial node. 
\begin{lemma}
\label{lm:no_common_nbr}
If $G=(V,E)$ is not moral, then $H=G+uv$ is not moral for any pair of non-adjacent $u,v$ s.t. $N_G(u)\cap N_G(v)=\emptyset$. 
\end{lemma}
\begin{proof}
$G$ is not moral implies the following two cases:

\textit{Case 1:} $D(G)\neq \emptyset$. The only possibility for
turning a node $x$ into a simplicial node in $H$ is when
$D_G(x)=uv$. This contradicts the premise
$N_G(u)\cap N_G(v)=\emptyset$. Hence, $D(H)\neq \emptyset$.

\textit{Case 2:} $D(G)=\emptyset$. Then $G$ has only partial PEKs,
each of which can lead to a subgraph $G^i \subset G$ s.t.
$D(G^i)\neq \emptyset$. To make a node $x \in G^i$ simplicial, either
$D_{G^i}(x)=uv$ for $u, v \in G^i$ or $u \notin G^i$ is a locally
simplicial node in $G$ s.t. $N_G(u) \cap V(G^i) = x$ and
$xv \in E(G^i)$. Both conditions, however, contradict to 
$N_G(u)\cap N_G(v)=\emptyset$. Hence, $D(H)\neq \emptyset$.
\end{proof}
The above lemma states that if a graph is not moral, adding an edge
between non-adjacent nodes who have no common neighbours will not
make it moral. The next lemma states that if $x$ is a simplicial node
s.t. no pair of its neighbours have a common neighbour outside of
$N_G[x]$, then morality is preserved after removing $x$ and all the edges
between its neighbours. 

\begin{lemma}
\label{lm:1kp}
Let $G=(V,E)$ be a moral graph. If $\exists x \in V$ with $D_G(x)=\emptyset$ s.t. for each pair $u, v \in N_G(x)$, their common neighbours $N_G(u)\cap N_G(v) \subset N_G[x]$, then $G'=G-x-E(G[N_G(x)])$ is moral.
\end{lemma}
\begin{proof}
  Assume $G'$ is not moral. The removal of $E(G[N_G(x)])$ implies
  every pair of $x$'s neighbours $u,v$ are non-adjacent in $G'$. In
  addition, $N_G(u)\cap N_G(v) \subset N_G[x]$ implies
  $N_{G'}(u)\cap N_{G'}(v)=\emptyset$. By Lemma
  \ref{lm:no_common_nbr}, for any non-empty proper subset
  $S \subsetneq E(G[N_G[x]])$, the subgraph $G''=G-x-S$ is not
  moral. It is not difficult to see that $G-x$ is not moral either,
  for otherwise $G'$ must be moral too. Hence, for any ordering $\alpha$ of $G$, $\nexists \epsilon_{\alpha}(x) \subseteq E(G[N_G(x)])$ s.t. the subgraph $G-x-\epsilon_{\alpha}(x)$ is moral. This contradicts to $G$ being moral, so $G'$ must be moral. 
\end{proof}

Based on Lemma \ref{lm:1kp}, we can prove that the morality of maximum degree $3$ graphs can be checked by recursively removing a simplicial node and all the edges between its neighbours. 
\begin{lemma}
\label{lm:wrs_deg3}
Let $G=(V,E)$ be a moral graph with $\Delta(G)=3$. If $\exists x \in V$ with $D_G(x)=\emptyset$, then $G'=G-x-E(G[N_G(x)])$ is moral. 
\end{lemma}
\begin{proof}
The cases when $d_G(x)=1$ or $3$ are trivial, because the former implies $x$ is a leave and the latter implies $G=K_4$ is a complete graph over $4$ nodes. 
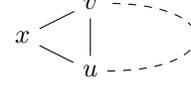
\begin{figure}
\centering
\begin{tikzpicture}[scale=0.9]
\begin{scope}           
    \node (a) at (0,0.5) {$x$};
    \node (b) at (1,0) {$u$};
    \node (c) at (1,1) {$v$};  
    \path [-] (a) edge (c);
    \path [-] (a) edge (b);
    \path [-] (c) edge (b);
\end{scope}
\draw[dashed] (b) .. controls (3,0) and (3,1) .. (c);
\end{tikzpicture}
\caption{A graph $G$ with $\Delta(G)=3$ and $D_G(x)=\emptyset$.}
\label{fg:deg3_1k3}
\end{figure}
For the case when $d_G(x)=2$, assume $N_G(x)=\{u,v\}$ (Figure \ref{fg:deg3_1k3}). If the edge $uv$ is not in a cycle in the subgraph $G-x$, then
$G'$ is moral. Suppose $uv$ is in a cycle in $G-x$. If $N_G(u) \cap N_G(v) = \{x,y\}$ s.t. $x \neq y$, then the subgraph $H=G-\{x,u,v,y\}$ must be moral. This is because $d_G(u)=d_G(v)=\Delta(G)$, so $H$ is connected to the rest of the graph via $y$ only. If $N_G(u) \cap N_G(v) = \{x\}$, by Lemma \ref{lm:1kp} $G'$ is moral. 
\end{proof}
\begin{algorithm}[]
\caption{Checking morality for maximum degree 3 graphs}
\label{alg:wrs_deg3}
\begin{algorithmic}[]
	\STATE {\bfseries Input:} graph $G=(V,E)$ s.t. $\Delta(G)=3$ 
    \WHILE{$\exists x$ s.t. $D_G(x)=\emptyset$}    	
    	\STATE {$G=G-x-E(G[N_G(x)])$}
    \ENDWHILE
    \IF{$G=\emptyset$}    
    	\STATE{\bfseries return}{ T} 
    \ELSE 
    	\STATE{\bfseries return}{ F}
    \ENDIF
\end{algorithmic}
\end{algorithm}

\begin{theorem}
\label{thm:wrs_deg3}
The morality of maximum degree $3$ graphs can be decided in polynomial time. 
\end{theorem}
\begin{proof}
  A straightforward algorithm (Algorithm \ref{alg:wrs_deg3}) for
  checking morality for maximum degree $3$ graphs can be deduced
  directly from Lemma \ref{lm:wrs_deg3}. The algorithm returns T when it reaches the empty graph, otherwise it returns F when stucking at a non-empty subgraph that has no simplicial node.

  A graph $G$ with $n$ nodes can be represented by an adjacency
  list, from which it takes polynomial time to find $N(x)$ for
  $x \in V$. Since $|N(x)| \le \Delta(G)=3$, it also takes polynomial time to
  verify $D(x)=\emptyset$. So the time complexity of finding a simplicial node is polynomial. The operations of removing $x$,
  $\{xy \in E\mid \forall y \in N_G(x)\}$ and $E(G[N(x)])$ take
  constant time. The while loop repeats at most $n$ times, so
  Algorithm \ref{alg:wrs_deg3} runs in polynomial time.
\end{proof}
The rest of this section focuses on graphs with maximum degree
$4$. Simplicial nodes in these graphs are treated differently in a
fixed order, depending on their degrees. Once simplicial nodes
satisfying certain conditions are removed, there are no other
simplicial nodes that satisfy the same conditions. First, we get rid
of simplicial nodes with degrees $1,3$ and $4$.

\begin{lemma}
\label{lm:deg1_4_3}
Let $G=(V,E)$ be a moral graph with $\Delta(G)=4$. If $\exists x \in V$ with $D_G(x)=\emptyset$ and $d_G(x)\in \{1,3,4\}$, then $G'=G-x-E(G[N(x)])$ is moral. 
\end{lemma}
\begin{proof}
  For $x \in V$ with $D_G(x)=\emptyset$, if $d_G(x)=1$ then $x$ is a
  leaf. If $d_G(x)=3$, the case is similar as having a degree $2$ simplicial node in a maximum degree $3$ graph shown in 
  Lemma \ref{lm:wrs_deg3}. If $d_G(x)=4$, the graph $G=K_5$ is a complete graph over $5$ nodes. Therefore, $G'$ is moral. 
\end{proof}

Next, we deal with degree $2$ simplicial nodes.Let $K_3^m$ denote a maximal stack of $m$ $K_3$s for $m \ge 1$. Maximal indicates that the length $m$ cannot be increased by adding more nodes in the stack. For example, Figure \ref{fg:2k3s} contains $K_3^2$ a maximal stack of $2$ $K_3$s. Corollary \ref{cor:1k3} is a special case of Lemma \ref{lm:1kp} when $x$ is a simplicial node in $K_3 \subset G$. 
\begin{corollary}
\label{cor:1k3}
Let $G=(V,E)$ be a moral graph with $\Delta(G)=4$. If $\exists x \in K_3^1 \subset G$ with $D_G(x)=\emptyset$ and $d_G(x)=2$, then $G'=G-x-E(G[N(x)])$ is moral.
\end{corollary}
\begin{proof}
This follows from Lemma \ref{lm:1kp}. 
\end{proof}

The next lemma states how morality can be preserved when dealing with simplicial nodes in $K_3^2$. 
\begin{lemma}
\label{lm:2k3s}
Let $G=(V,E)$ be a moral graph with $\Delta(G)=4$. If $\exists x \in K_3^2 \subset G$ with $D_G(x)=\emptyset$ and $d_G(x)=2$, then $G'=G-x$ is moral.
\begin{figure}
\centering
\begin{tikzpicture}[scale=0.9]
\begin{scope}[]               
    \node (a) at (0,0) {$v_1$};
    \node (b) at (1,0) {$v_2$};
    \node (c) at (1,1) {$v_3$};  
    \node (d) at (2,0) {$v_4$};
        
    \path [-] (a) edge (c);
    \path [-] (a) edge (b);
    \path [-] (b) edge (d);
    \path [-] (c) edge (b);
    \path [-] (c) edge (d);
    \draw[dashed] (d) -- (2.7,0.5);
    \draw[dashed] (d) -- (3,0);        
\end{scope}
\draw[dashed] (b) .. controls (4,-2) and (5,1) .. (c);
\end{tikzpicture}
\vskip -0.4in 
\caption{A graph $G$ with $\Delta(G)=4$ and $K_3^2\subset G$.}
\label{fg:2k3s}
\end{figure}
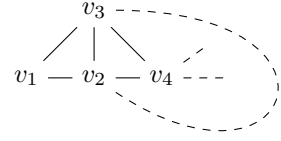
\end{lemma}
\begin{proof}
Suppose $G$ is labelled as shown in Figure \ref{fg:2k3s}, where $x=v_1$. Assuming $G'$ is not moral. It implies $G'$ has no PEK. Let $G''=G-v_1-v_2v_3$. The removal of the edge $v_2v_3$ implies that $v_4$ in $G''$ cannot be locally simplicial before $v_2$ or $v_3$. It also implies $v_3v_4 \notin \epsilon_{\alpha}(v_2)$ or $v_2v_4 \notin \epsilon_{\alpha}(v_3)$, if $v_2$ or $v_3$ ever becomes locally simplicial in $G''$ for an ordering $\alpha$. Hence, the space of all orderings of $G''$ is a subspace of all orderings of $G'$. And for any local simplicial node, its excess in $G''$ has no more options than in $G'$. Therefore, if $G'$ has no PEK, then $G''$ has no PEK either. This contradicts to $G$ being moral.
\end{proof}

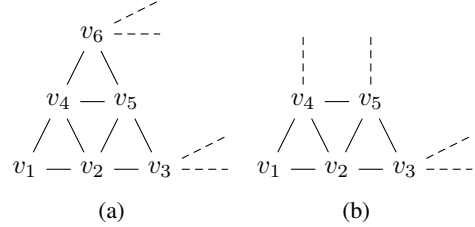
\begin{figure}
\centering
\subfigure[]{
\label{fg:3k3s_dist_1}
\begin{tikzpicture}[scale=0.9]
	\node (A) at (-0.5+8,-2-3-0.5) {$v_1$};
    \node (B) at (1.5+8,-2-3-0.5) {$v_3$};
    \node (F) at (0.5+8,0-3-0.5) {$v_6$};
    \node (H) at (0+8,-1-3-0.5) {$v_4$};
    \node (I) at (1+8,-1-3-0.5) {$v_5$};
    \node (J) at (0.5+8,-2-3-0.5) {$v_2$};  
	\path [-] (A) edge (J);
    \path [-] (A) edge (H);
    \path [-] (B) edge (I);
    \path [-] (B) edge (J);
    \path [-] (F) edge (H);
    \path [-] (F) edge (I);
    \path [-] (H) edge (I);
    \path [-] (H) edge (J);
    \path [-] (I) edge (J);
    \path [densely dashed] (F) edge (9.5,-3);
    \path [densely dashed] (F) edge (9.5,-3.5);
    \path [densely dashed] (B) edge (10.5,-5);    
    \path [densely dashed] (B) edge (10.5,-5.5);    
\end{tikzpicture}
}
\subfigure[]{
\label{fg:3k3s_dist_infty}
\begin{tikzpicture}[scale=0.9]
    \node (A) at (-0.5+1,-2) {$v_1$};
    \node (B) at (1.5+1,-2) {$v_3$};
    \node (H) at (0+1,-1) {$v_4$};
    \node (I) at (1+1,-1) {$v_5$};
    \node (J) at (0.5+1,-2) {$v_2$};               
    \path [-] (A) edge (J);
    \path [-] (A) edge (H);
    \path [-] (B) edge (I);
    \path [-] (B) edge (J);
    \path [-] (H) edge (I);
    \path [-] (H) edge (J);
    \path [-] (I) edge (J); 
    \path [densely dashed] (H) edge (1,0);
    \path [densely dashed] (I) edge (2,0);
    \path [densely dashed] (B) edge (3.5,-2);
    \path [densely dashed] (B) edge (3.5,-1.5);
\end{tikzpicture}
}
\caption{Two graphs $G$ with $\Delta(G)=4$ and $K_3^3 \subset G$ s.t. the distance $d(v_4,v_5)\in \{2, \infty\}$ in $G-\{v_1,v_2,v_3\}-v_4v_5$.}
\label{fg:3k3s_dist_2_infty}
\end{figure}

The following three lemmas consider simplicial nodes that are in $K_3^3$. Within this case, simplicial nodes are treated differently, depending on the distance $d(v_4, v_5)$ in $G-\{v_1,v_2,v_3\}-v_4v_5$ as shown in Figure \ref{fg:3k3s_dist_2_infty} and \ref{fg:3k3s_dist_3}. 
\begin{lemma}
\label{lm:3k3s_dist_2_infty}
Let $G=(V,E)$ be a moral graph with $\Delta(G)=4$. If there is a simplicial node $v_1$ as shown in Figure \ref{fg:3k3s_dist_2_infty}, then $G'=G-v_1$ is moral.
\end{lemma}
\begin{proof}
  Assume $G'$ is not moral. Let $G''=G'+v_2v_4$. The addition of
  the edge $v_2v_4$ makes a 3-clique over $\{v_2,v_4,v_5\}$. But a
  clique is only critical for turning $G''$ into a moral graph if it
  can break unbreakable cycles in $G'$. However,
  $d_G(v_2)=d_G(v_4)=\Delta(G)$ in Figure \ref{fg:3k3s_dist_1} implies
  that the 3-clique does not share edges with any cycles that could
  appear in the subgraph $G-\{v_1,\dots,v_6\}$. In Figure
  \ref{fg:3k3s_dist_infty}, $d_G(v_2)=\Delta(G)$ and $d(v_4,v_5)=\infty$ leads to the same conclusion. Hence, $G''$
  is not moral. This contradicts to $G$ being moral. 
\end{proof}

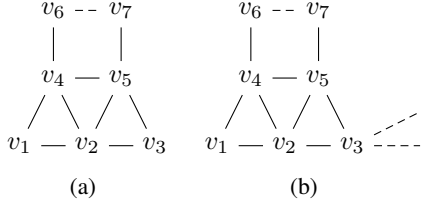
\begin{figure}
\centering
\subfigure[]{
\label{fg:3k3s_dist_3_1}
\begin{tikzpicture}[scale=0.9]
    \node (A) at (-0.5+1,-2) {$v_1$};
    \node (B) at (1.5+1,-2) {$v_3$};
    \node (H) at (0+1,-1) {$v_4$};
    \node (I) at (1+1,-1) {$v_5$};
    \node (J) at (0.5+1,-2) {$v_2$};     
    \node (C) at (1, 0) {$v_6$};
    \node (D) at (2, 0) {$v_7$};
    \path [-] (A) edge (J);
    \path [-] (A) edge (H);
    \path [-] (B) edge (I);
    \path [-] (B) edge (J);
    \path [-] (H) edge (I);
    \path [-] (H) edge (J);
    \path [-] (I) edge (J); 
    \path [-] (H) edge (C);
    \path [-] (I) edge (D);
    \path [densely dashed] (C) edge (D);
\end{tikzpicture}
}
\subfigure[]{
\label{fg:3k3s_dist_3_2}
\begin{tikzpicture}[scale=0.9]
    \node (A) at (-0.5+1,-2) {$v_1$};
    \node (B) at (1.5+1,-2) {$v_3$};
    \node (H) at (0+1,-1) {$v_4$};
    \node (I) at (1+1,-1) {$v_5$};
    \node (J) at (0.5+1,-2) {$v_2$};     
    \node (C) at (1, 0) {$v_6$};
    \node (D) at (2, 0) {$v_7$};
    \path [-] (A) edge (J);
    \path [-] (A) edge (H);
    \path [-] (B) edge (I);
    \path [-] (B) edge (J);
    \path [-] (H) edge (I);
    \path [-] (H) edge (J);
    \path [-] (I) edge (J); 
    \path [-] (H) edge (C);
    \path [-] (I) edge (D);
    \path [densely dashed] (C) edge (D);
    \path [densely dashed] (B) edge (3.5,-2);
    \path [densely dashed] (B) edge (3.5,-1.5);
\end{tikzpicture}
}
\caption{Two graphs $G$ with $\Delta(G)=4$ and $K_3^3 \subset G$ s.t. the distance $d(v_4,v_5) \in [3, \infty)$ in $G-\{v_1,v_2,v_3\}-v_4v_5$.}
\label{fg:3k3s_dist_3}
\end{figure}

\begin{lemma}
Let $G=(V,E)$ be a moral graph with $\Delta(G)=4$. If there are two simplicial nodes $v_1,v_3$ as shown in Figure \ref{fg:3k3s_dist_3_1}, then $G'=G-\{v_1,v_3\}$ is moral.
\end{lemma}
\begin{proof}
The proof is trivial. 
\end{proof}

\begin{lemma}
\label{lm:3k3s_dist_3}
Let $G=(V,E)$ be a moral graph with $\Delta(G)=4$. If there is a simplicial node $v_1$ as shown in Figure \ref{fg:3k3s_dist_3_2}, then  $G'=G-v_1-E(G[N_G(v_1)])$ is moral. 
\end{lemma}
\begin{proof}
There is only one simplicial node in each $K_3^3$ and all simplicial nodes are in the same condition as $v_1$. Removing $v_1$ does not introduce new simplicial nodes in the subgraph. Hence, if $G$ is moral, $G-v_1-v_2v_4$ must be moral too. 
\end{proof}

The next lemma states how a long stack of $K_3^m$ can be shortened while morality is still preserved. The length of the stack is decreased by two at a time untill it becomes $1,2$ or $3$ that can be dealth with using prior rules. 
\begin{lemma}
Let $G=(V,E)$ be a moral graph with $\Delta(G)=4$. If $\exists x \in K_3^m \subset G$ for $m > 3$ with $D_G(x)=\emptyset$ and $d_G(x)=2$, then $G'=G-x-E(G[N_G(x)])$ is moral.
\end{lemma}
\begin{proof}
For $m > 3$, only the two nodes on each end of a $K_3^m$ have degrees less than $\Delta(G)$. Hence, none of the 3-cliques in the middle of a $K_3^m$ shares an edge with a cycle in $G$, so $G'$ remains moral. 
\end{proof}


\begin{algorithm}[]
\caption{Checking morality for maximum degree $4$ graphs}
\label{alg:d_wrs_deg4}
\begin{algorithmic}[]
	\STATE {\bfseries Input:} graph $G=(V,E)$ s.t. $\Delta(G)=4$ 
    
    \IF{$\exists x$ s.t. $D_G(x)=\emptyset,d_G(x)=4$}
       \STATE {\bfseries return}{ T}
    \ENDIF
    
    \WHILE {$D(G)=\emptyset$} 
    
    \IF{$\exists x$ s.t. $D_G(x)=\emptyset, d_G(x)=1$}
       \STATE $G=G-x$
    \ELSIF{$\exists x$ s.t. $D_G(x)=\emptyset,d_G(x)=3$}
       \STATE $G=G-x-E(G[N_G(x)])$
    \ELSIF{$\exists x \in K_3^m$ s.t. $D_G(x)=\emptyset$ for $m \ge 4$} 
       \STATE $G=G-x-E(G[N_G(x)])$
    \ELSIF{$\exists x \in K_3^1$ s.t. $D_G(x)=\emptyset$}
       \STATE $G=G-x-E(G[N_G(x)])$
    \ELSIF{$\exists x \in K_3^2$ s.t. $D_G(x)=\emptyset$}
       \STATE $G=G-x$
    \ELSE 
    	\IF{$d(v_4,v_5) \in \{2, \infty\}$ in $G-\{v_1,v_2,v_3\}-v_4v_5$} 
    		\STATE $G=G-x-E(G[N_G(x)])$
    	\ELSE 
    		\IF{$\exists y \in K_3^3$ s.t. $D_G(y)=\emptyset$ and $|N_G(x)\cap N_G(y)|=1$} 
       			\STATE $G=G-\{x,y\}$
   			\ELSE 
       			\STATE $G=G-x-E(G[N_G(x)])$
       		\ENDIF
    	\ENDIF
    \ENDIF
    
    \ENDWHILE
    
    \IF{$G=\emptyset$}    
    	\STATE{\bfseries return}{ T} 
    \ELSE 
    	\STATE{\bfseries return}{ F}
    \ENDIF
    
\end{algorithmic}
\end{algorithm}

\begin{theorem}
\label{thm:deg4}
The morality of maximum degree $4$ graphs can be checked in polynomial time.
\end{theorem}
\begin{proof}
The correctness of Algorithm \ref{alg:d_wrs_deg4} can be proved by the above lemmas and corollary. 

The complexity of this algorithm is mainly determined by identifying simplicial nodes in different scenarios. The worst case is the identification of a simplicial node in a long $K_3^m$. This, however, is still bounded in polynomial time, because once a $K_3^m$ is confirmed to have length greater than $3$, the actual length does not matter anymore. If a $K_3^3$ is matched, $d(v_4,v_5$) can be calculated in $O(n^2)$ time (using Dijkstra's
algorithm). The rest of the operations can all be done in polynomial time. Hence, the algorithm has a polynomial time complexity. 
\end{proof}

As mentioned earlier, a moral graph's simplicial nodes need to be removed in a fixed order in order for it to be completely eliminated. Figure \ref{fg:fixed_order_example} shows two examples of moral graphs that cannot be completely eliminated if simplicial nodes are removed in a different order. 
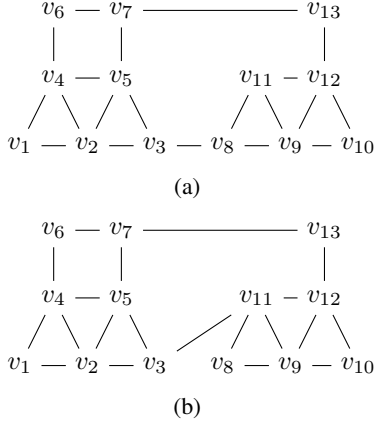
\begin{figure}
\centering
\subfigure[]{
\label{fg:fixed_order_example_1}
\begin{tikzpicture}[scale=0.9]
\begin{scope}              
    \node (A) at (-0.5,-2) {$v_1$};
    \node (B) at (1.5,-2) {$v_3$};
    \node (H) at (0,-1) {$v_4$};
    \node (I) at (1,-1) {$v_5$};
    \node (J) at (0.5,-2) {$v_2$};  
    \node (C) at (0,0) {$v_6$};
    \node (D) at (1,0) {$v_7$};    
    \path [-] (A) edge (J);
    \path [-] (A) edge (H);
    \path [-] (B) edge (I);
    \path [-] (B) edge (J);
    \path [-] (H) edge (I);
    \path [-] (H) edge (J);
    \path [-] (I) edge (J); 
    \path [-] (H) edge (C);
    \path [-] (C) edge (D);
    \path [-] (I) edge (D);
    
    \node (a) at (-0.5+3,-2) {$v_8$};
    \node (b) at (1.5+3,-2) {$v_{10}$};
    \node (h) at (0+3,-1) {$v_{11}$};
    \node (i) at (1+3,-1) {$v_{12}$};
    \node (j) at (0.5+3,-2) {$v_9$};
    \node (d) at (1+3,0) {$v_{13}$};
    \path [-] (a) edge (j);
    \path [-] (a) edge (h);
    \path [-] (b) edge (i);
    \path [-] (b) edge (j);
    \path [-] (h) edge (i);
    \path [-] (h) edge (j);
    \path [-] (i) edge (j);
    
    \path [-] (a) edge (B);
    \path [-] (D) edge (d);
    \path [-] (d) edge (i);
\end{scope}
\end{tikzpicture}
}
\subfigure[]{
\label{fg:fixed_order_example_2}	
\begin{tikzpicture}[scale=0.9]
\begin{scope}              
    \node (A) at (-0.5+6.5,-2) {$v_1$};                                                                                                                                                                                                                       
    \node (B) at (1.5+6.5,-2) {$v_3$};
    \node (H) at (0+6.5,-1) {$v_4$};
    \node (I) at (1+6.5,-1) {$v_5$};
    \node (J) at (0.5+6.5,-2) {$v_2$};  
    \node (C) at (0+6.5,0) {$v_6$};
    \node (D) at (1+6.5,0) {$v_7$};    
    \path [-] (A) edge (J);
    \path [-] (A) edge (H);
    \path [-] (B) edge (I);
    \path [-] (B) edge (J);
    \path [-] (H) edge (I);
    \path [-] (H) edge (J);
    \path [-] (I) edge (J); 
    \path [-] (H) edge (C);
    \path [-] (C) edge (D);
    \path [-] (I) edge (D);
    
    \node (a) at (-0.5+3+6.5,-2) {$v_8$};
    \node (b) at (1.5+3+6.5,-2) {$v_{10}$};
    \node (h) at (0+3+6.5,-1) {$v_{11}$};
    \node (i) at (1+3+6.5,-1) {$v_{12}$};
    \node (j) at (0.5+3+6.5,-2) {$v_9$};
    \node (d) at (1+3+6.5,0) {$v_{13}$};
    \path [-] (a) edge (j);
    \path [-] (a) edge (h);
    \path [-] (b) edge (i);
    \path [-] (b) edge (j);
    \path [-] (h) edge (i);
    \path [-] (h) edge (j);
    \path [-] (i) edge (j);
    
    \path [-] (h) edge (B);
    \path [-] (D) edge (d);
    \path [-] (d) edge (i);
\end{scope}
\end{tikzpicture}
}
\caption{Two maximum degree $4$ moral graphs with simplicial nodes in $K^3_3$. According to Algorithm \ref{alg:d_wrs_deg4}, in \ref{fg:fixed_order_example_1} $\{v_{10}, v_9v_{12}\}$ are removed before $\{v_{1}, v_2v_4\}$; in \ref{fg:fixed_order_example_2} $\{v_8,v_{10}\}$ are removed before $\{v_1,v_2v_4\}$. If the order is not followed, these graphs will not be recognized as moral by Algorithm \ref{alg:d_wrs_deg4}.}
\label{fg:fixed_order_example}
\end{figure}

To this point, we have proved that for graphs with maximum degree $3$
and $4$, their morality can be checked in polynomial time. The next
theorem proves that the problem remains NP-complete for graphs with
maximum degree $5$, and hence the same for graphs with even higher
maximum degrees.
\begin{theorem}
\label{thm:deg5}
The problem of checking morality for maximum degree $5$ graphs is NP-complete. 
\end{theorem}
The theorem can be proved by modifying \citet{verma1993deciding}'s construction to build graphs with max degree $5$. 


\begin{proof}
Given a 3-CNF problem with $n$ variables and $t$ clauses, our
construction will build a graph with $32n+23t+7$ vertices, which are
made of $32$ vertices in each of the $n$ variable gadgets, $22$
vertices in each of the $t$ clause gadgets and $7+t$ vertices in the
auxiliary gadget. The variable (Figure \ref{fg:variable_gedget}) and
clause (Figure \ref{fg:clause_gedget}) gadgets are identical to those
used by \citet{verma1993deciding}, but the auxiliary gadget (Figure
\ref{fg:aux_gedget}) now consists of a chain of length $t+2$, each of $S_i^7$ in which connects to a clause gadget. This avoids having a single node $S^7$ connects to all clause gadgets, which results in high node degree as
appeared in Figure $4$ in \cite{verma1993deciding}. 

The gadgets are connected together to form a single component in the following ways: 
\begin{enumerate}
\item all the variable gadgets are connected together by the edges $\bar{v}_i^0v_{i+1}^0$ for $i \in [1, n-1]$,
\item the variable gadgets are connected to the auxiliary gadget by $S^0v_1^0$ and $S^5\bar{v}_n^0$,
\item the clause gadgets are connected to the auxiliary gadget by $S_i^7F_i^{21}$ for $i \in [1, t]$,
\item for the $k^{th}$ clause, if its $(l+1)^{th}$ literal is the variable $v_i$ for $l \in [0, 2]$, 
\begin{enumerate}
\item if $d(v_i^{15})=1$, then $F_k^l$ is connected to $v_i^{15}$,
\item else $F_k^l$ is connected to $F_p^q$ and $F_p^{q+3}$ for the last $F_p^q$ that was (directly or indirectly) connected to $v_i^{15}$,
\end{enumerate}
\item if the $(l+1)^{th}$ literal in the $k^{th}$ clause is $\bar{v}_i$, replacing $v_i^{15}$ by $\bar{v}_i^{15}$ in step $4$.
\end{enumerate}
Steps $1-3$ are identical to those in \cite{verma1993deciding}. Steps $4$ and $5$ are different in order to avoid high degree nodes  $v_i^{15}$ and  $\bar{v}_1^{15}$. Figure \ref{fg:3cnf_fully_dir} is an example of a construction from a satisfiable 3-CNF. The reduction from 3-CNF to graph morality is polynomial. It remains to show that the two problems are equivalent.

The above construction ensures the final graph will always have simplicial nodes $\{v_i^7, \bar{v}_i^7, v_i^9,\bar{v}_i^9\}$. For the graph to be maximally eliminated for any ordering $\alpha$, either $(i)$ $\epsilon_{\alpha}(v_i^7)=v_i^8v_i^{10}$, $\epsilon_{\alpha}(v_i^9)=v_i^8v_i^{11}$, $\epsilon_{\alpha}(\bar{v}_i^7)=\epsilon_{\alpha}(\bar{v}_i^9)=\emptyset$ or $(ii)$ $\epsilon_{\alpha}(\bar{v}_i^7)=\bar{v}_i^8\bar{v}_i^{10}$, $\epsilon_{\alpha}(\bar{v}_i^9)=\bar{v}_i^8\bar{v}_i^{11}$, $\epsilon_{\alpha}(v_i^7)=\epsilon_{\alpha}(v_i^9)=\emptyset$. Hence, a variable is assigned T or F according to how its corresponding gadget is eliminated according to these two choices. 

If the graph is not moral, there is a clause gadget that cannot be eliminated because no elimination can get to it through any variable gadget. Therefore, no matter how the variables are assigned, this clause only returns F in the expression, so the 3-CNF is not satisfiable. If the graph is moral, assigning F if a variable gadget is eliminated by excess $(i)$ and T otherwise. Therefore, each clause gadget contains a true literal, so the 3-CNF is satisfiable. 
\end{proof}

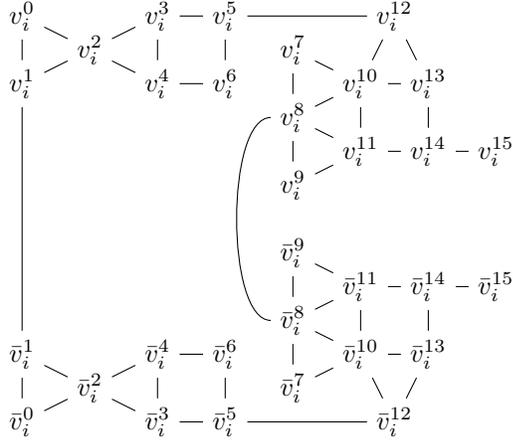
\begin{figure}
\centering
\begin{tikzpicture}[scale=0.9]
\begin{scope}
    \node (A) at (0,3.5) {$v_i^0$};
    \node (B) at (0,2.5) {$v_i^1$};
    \node (C) at (1,3) {$v_i^2$};
    \node (D) at (2,3.5) {$v_i^3$};
    \node (E) at (2,2.5) {$v_i^4$};
    \node (F) at (3,3.5) {$v_i^5$};
    \node (G) at (3,2.5) {$v_i^6$};
    \node (H) at (4,3) {$v_i^7$};
    \node (I) at (4,2) {$v_i^8$};
    \node (J) at (4,1) {$v_i^9$};
    \node (K) at (5,2.5) {$v_i^{10}$};
    \node (L) at (5,1.5) {$v_i^{11}$};
    \node (M) at (5.5,3.5) {$v_i^{12}$};
    \node (N) at (6,2.5) {$v_i^{13}$};
    \node (O) at (6,1.5) {$v_i^{14}$};
    \node (P) at (7,1.5) {$v_i^{15}$};
    
    \node (a) at (0,-2.5) {$\bar{v}_i^0$};
    \node (b) at (0,-1.5) {$\bar{v}_i^1$};
    \node (c) at (1,-2) {$\bar{v}_i^2$};
    \node (d) at (2,-2.5) {$\bar{v}_i^3$};
    \node (e) at (2,-1.5) {$\bar{v}_i^4$};
    \node (f) at (3,-2.5) {$\bar{v}_i^5$};
    \node (g) at (3,-1.5) {$\bar{v}_i^6$};
    \node (h) at (4,-2) {$\bar{v}_i^7$};
    \node (i) at (4,-1) {$\bar{v}_i^8$};
    \node (j) at (4,0) {$\bar{v}_i^9$};
    \node (k) at (5,-1.5) {$\bar{v}_i^{10}$};
    \node (l) at (5,-0.5) {$\bar{v}_i^{11}$};
    \node (m) at (5.5,-2.5) {$\bar{v}_i^{12}$};
    \node (n) at (6,-1.5) {$\bar{v}_i^{13}$};
    \node (o) at (6,-0.5) {$\bar{v}_i^{14}$};
    \node (p) at (7,-0.5) {$\bar{v}_i^{15}$};
\end{scope}

\begin{scope}[>={Stealth[black]},
              every edge/.style={draw=black}]
    \path [-] (B) edge (b);
    \path [-] (A) edge (B);
    \path [-] (A) edge (C);
    \path [-] (B) edge (C);
    \path [-] (C) edge (D);
    \path [-] (C) edge (E);
    \path [-] (D) edge (E);
    \path [-] (D) edge (F);
    \path [-] (E) edge (G);
    \path [-] (F) edge (G);
    \path [-] (F) edge (M);
    \path [-] (H) edge (I);
    \path [-] (H) edge (K);
    \path [-] (K) edge (I);
    \path [-] (I) edge (J);
    \path [-] (I) edge (L);
    \path [-] (L) edge (J);
    \path [-] (K) edge (M);
    \path [-] (K) edge (N);
    \path [-] (K) edge (L);
    \path [-] (M) edge (N);
    \path [-] (N) edge (O);
    \path [-] (L) edge (O);
    \path [-] (O) edge (P);
    
    \path [-] (a) edge (b);
    \path [-] (a) edge (c);
    \path [-] (b) edge (c);
    \path [-] (c) edge (d);
    \path [-] (c) edge (e);
    \path [-] (d) edge (e);
    \path [-] (d) edge (f);
    \path [-] (e) edge (g);
    \path [-] (f) edge (g);
    \path [-] (f) edge (m);
    \path [-] (h) edge (i);
    \path [-] (h) edge (k);
    \path [-] (k) edge (i);
    \path [-] (i) edge (j);
    \path [-] (i) edge (l);
    \path [-] (l) edge (j);
    \path [-] (k) edge (m);
    \path [-] (k) edge (n);
    \path [-] (k) edge (l);
    \path [-] (m) edge (n);
    \path [-] (n) edge (o);
    \path [-] (l) edge (o);
    \path [-] (o) edge (p);
    
\end{scope}
\draw[-] (I) .. controls (3,2) and (3,-1) .. (i);
\end{tikzpicture}
\caption{A gadget that simulates the behaviour of a boolean variable $v_i$. It consists of two symmetric parts, $v_i$ (top) and $\bar{v}_i$ (bottom) that are connected by the edge $v_i^8\bar{v}_i^8$. This single edge guarantees non-identical excesses $\epsilon_{\alpha}(v_i^j)\neq \epsilon_{\alpha}(\bar{v}_i^j)$ for some $j \in [0,15]$ for any ordering $\alpha$ of the gadget, so the two parts must be oriented differently to distinguish between T and F.}
\label{fg:variable_gedget}
\end{figure}

\begin{figure}
\centering
\begin{tikzpicture}[scale=0.9]
\begin{scope}
    \node (a) at (0,4.5) {$F_i^0$};
    \node (b) at (0,2.5) {$F_i^1$};   
    \node (c) at (0,0.5) {$F_i^2$};
    \node (d) at (1,4.5) {$F_i^3$};
    \node (e) at (1,2.5) {$F_i^4$};
    \node (f) at (1,0.5) {$F_i^5$};
    \node (g) at (2,5) {$F_i^6$};
    \node (h) at (2,4) {$F_i^7$};
    \node (i) at (2,3) {$F_i^8$};
    \node (j) at (2,2) {$F_i^9$};
    \node (k) at (2,1) {$F_i^{10}$};
    \node (l) at (2,0) {$F_i^{11}$};
    \node (m) at (3,5) {$F_i^{12}$};
    \node (n) at (3,4) {$F_i^{13}$};
    \node (o) at (3,3) {$F_i^{14}$};
    \node (p) at (3,2) {$F_i^{15}$};
    \node (q) at (3,1) {$F_i^{16}$};
    \node (r) at (3,0) {$F_i^{17}$};
    \node (s) at (4,3) {$F_i^{19}$};
    \node (t) at (6,5) {$F_i^{18}$};
    \node (u) at (6,1) {$F_i^{20}$};
    \node (v) at (8,3) {$F_i^{21}$};
\end{scope}

\begin{scope}[>={Stealth[black]},every edge/.style={draw=black}]
    \path [-] (a) edge (d);
    \path [-] (d) edge (g);
    \path [-] (d) edge (h);
    \path [-] (g) edge (h);
    \path [-] (g) edge (m);
    \path [-] (h) edge (n);
    \path [-] (m) edge (n);
    \path [-] (m) edge (t);
    
    \path [-] (b) edge (e);
    \path [-] (i) edge (e);
    \path [-] (j) edge (e);
    \path [-] (i) edge (j);
    \path [-] (i) edge (o);
    \path [-] (j) edge (p);
    \path [-] (o) edge (p);
    \path [-] (o) edge (s);
    \path [-] (s) edge (u);
    \path [-] (s) edge (v);
    \path [-] (v) edge (u);
    \path [-] (t) edge (u);
    \path [-] (t) edge (v);
    \path [-] (t) edge (s);
    
    \path [-] (c) edge (f);
    \path [-] (f) edge (k);
    \path [-] (f) edge (l);
    \path [-] (k) edge (l);
    \path [-] (k) edge (q);
    \path [-] (l) edge (r);
    \path [-] (q) edge (r);
    \path [-] (q) edge (u);
\end{scope}
\end{tikzpicture}
\caption{A gadget that simulates a clause $F_i$'s disjunction. It consists of a $K_4$ and three envelope graphs, one for each literal in $F_i$. One of $F_i^{18},F_i^{19}$ and $F_i^{20}$ can be locally simplicial, depending on whether or not its adjacent envelope graph corresponds to a T literal in $F_i$.}
\label{fg:clause_gedget}
\end{figure}
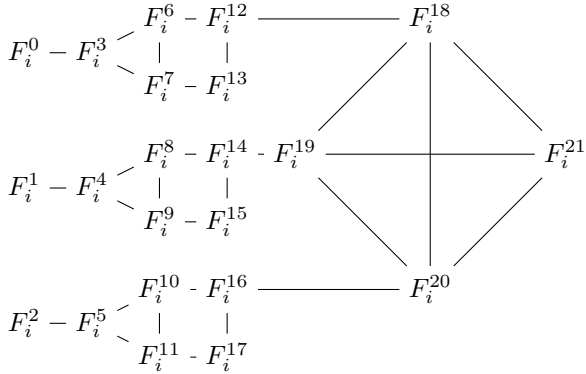

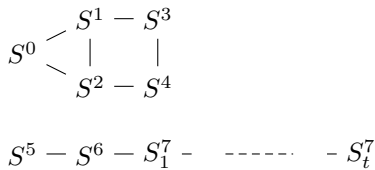
\begin{figure}
\centering
\begin{tikzpicture}[scale=0.9]
\begin{scope}
    \node (a) at (1,4.5) {$S^0$};
    \node (b) at (2,5) {$S^1$};   
    \node (c) at (2,4) {$S^2$};
    \node (d) at (3,5) {$S^3$};
    \node (e) at (3,4) {$S^4$};
    
    \node (f) at (1,3) {$S^5$};
    \node (g) at (2,3) {$S^6$};
    \node (h) at (3,3) {$S^7_1$};
    \node (i) at (6,3) {$S^7_t$};
    
    \path [-] (a) edge (b);
    \path [-] (a) edge (c);
    \path [-] (b) edge (c);
    \path [-] (b) edge (d);
    \path [-] (c) edge (e);
    \path [-] (d) edge (e);
    \path [-] (f) edge (g);
	\path [-] (h) edge (g);
	\path [-] (h) edge (3.5,3);
	\path [-] (5.5,3) edge (i);
	\path [densely dashed] (4,3) edge (5,3);
\end{scope}
\end{tikzpicture}
\caption{The auxiliary gadget consists of two parts. The chain connects the variable and clause gadgets to form a connected graph. The envelope graph is positioned on one side of the variable gadgets to enforce a certain direction.}
\label{fg:aux_gedget}
\end{figure}

\section{Conclusion}
In this paper, we have drawn a connection between checking Markov
blanket consistency and graph morality. We proved that being moral is
equivalent to being weakly recursively simplicial as well as having a
perfect elimination kit. We have also proved that checking morality
for maximum degree $3$ and $4$ graphs can be done in polynomial time,
but the problem remains NP-complete for graphs with higher maximum
degrees.

It is future work to develop an efficient way of enforcing
morality. This may produce a set of consistent Markov blankets that
can help with the performance of structure learning methods building
on Markov blankets. Another interesting possibility is
\textit{immoralizing} a moral graph to obtain a consistent DAG. This
could unify a (symmetric and consistent) set of Markov blankets to
obtain a DAG, one which may not be the generating model, but could be
used as a starting for heuristic structure learners.

\begin{figure}
\centering
\begin{tikzpicture}[scale=0.75]
\begin{scope}[every node/.style={circle,draw,fill=black,minimum size=1mm,inner sep=1pt},>={Stealth[black]}]
    \node (A) at (0,0.5) {};
    \node (B) at (0,1) {};
    \node (C) at (0.5,0.75) {};
    \node (D) at (1,0.5) {};
    \node (E) at (1,1) {};
    \node (F) at (1.5,0.5) {};
    \node (G) at (1.5,1) {};
    \node (H) at (2,0.75) {};
    \node (I) at (2,1.25) {};
    \node (J) at (2,1.75) {};
    \node (K) at (2.5,1) {};
    \node (L) at (2.5,1.5) {};
    \node (M) at (2.75,0.5) {};
    \node (N) at (3,1) {};
    \node (O) at (3,1.5) {};
    \node (P) at (3.5,1.5) {};
    \node (a) at (0,3.5) {};
    \node (b) at (0,3) {};
    \node (c) at (0.5,3.25) {};
    \node (d) at (1,3.5) {};
    \node (e) at (1,3) {};
    \node (f) at (1.5,3.5) {};
    \node (g) at (1.5,3) {};
    \node (h) at (2,3.25) {};
    \node (i) at (2,2.75) {};
    \node (j) at (2,2.25) {};
    \node (k) at (2.5,3) {};
    \node (l) at (2.5,2.5) {};
    \node (m) at (2.75,3.5) {};
    \node (n) at (3,3) {};
    \node (o) at (3,2.5) {};
    \node (p) at (3.5,2.5) {};
    
    \node (A1) at (0,0.5+3.5) {};
    \node (B1) at (0,1+3.5) {};
    \node (C1) at (0.5,0.75+3.5) {};
    \node (D1) at (1,0.5+3.5) {};
    \node (E1) at (1,1+3.5) {};
    \node (F1) at (1.5,0.5+3.5) {};
    \node (G1) at (1.5,1+3.5) {};
    \node (H1) at (2,0.75+3.5) {};
    \node (I1) at (2,1.25+3.5) {};
    \node (J1) at (2,1.75+3.5) {};
    \node (K1) at (2.5,1+3.5) {};
    \node (L1) at (2.5,1.5+3.5) {};
    \node (M1) at (2.75,0.5+3.5) {};
    \node (N1) at (3,1+3.5) {};
    \node (O1) at (3,1.5+3.5) {};
    \node (P1) at (3.5,1.5+3.5) {};
    \node (a1) at (0,3.5+3.5) {};
    \node (b1) at (0,3+3.5) {};
    \node (c1) at (0.5,3.25+3.5) {};
    \node (d1) at (1,3.5+3.5) {};
    \node (e1) at (1,3+3.5) {};
    \node (f1) at (1.5,3.5+3.5) {};
    \node (g1) at (1.5,3+3.5) {};
    \node (h1) at (2,3.25+3.5) {};
    \node (i1) at (2,2.75+3.5) {};
    \node (j1) at (2,2.25+3.5) {};
    \node (k1) at (2.5,3+3.5) {};
    \node (l1) at (2.5,2.5+3.5) {};
    \node (m1) at (2.75,3.5+3.5) {};
    \node (n1) at (3,3+3.5) {};
    \node (o1) at (3,2.5+3.5) {};
    \node (p1) at (3.5,2.5+3.5) {};
    
    \node (A2) at (0,0.5+7) {};
    \node (B2) at (0,1+7) {};
    \node (C2) at (0.5,0.75+7) {};
    \node (D2) at (1,0.5+7) {};
    \node (E2) at (1,1+7) {};
    \node (F2) at (1.5,0.5+7) {};
    \node (G2) at (1.5,1+7) {};
    \node (H2) at (2,0.75+7) {};
    \node (I2) at (2,1.25+7) {};
    \node (J2) at (2,1.75+7) {};
    \node (K2) at (2.5,1+7) {};
    \node (L2) at (2.5,1.5+7) {};
    \node (M2) at (2.75,0.5+7) {};
    \node (N2) at (3,1+7) {};
    \node (O2) at (3,1.5+7) {};
    \node (P2) at (3.5,1.5+7) {};
    \node (a2) at (0,3.5+7) {};
    \node (b2) at (0,3+7) {};
    \node (c2) at (0.5,3.25+7) {};
    \node (d2) at (1,3.5+7) {};
    \node (e2) at (1,3+7) {};
    \node (f2) at (1.5,3.5+7) {};
    \node (g2) at (1.5,3+7) {};
    \node (h2) at (2,3.25+7) {};
    \node (i2) at (2,2.75+7) {};
    \node (j2) at (2,2.25+7) {};
    \node (k2) at (2.5,3+7) {};
    \node (l2) at (2.5,2.5+7) {};
    \node (m2) at (2.75,3.5+7) {};
    \node (n2) at (3,3+7) {};
    \node (o2) at (3,2.5+7) {};
    \node (p2) at (3.5,2.5+7) {};

    \path [-] (B) edge (b);
    \path [-] (B) edge (A);
    \path [-] (C) edge (A);
    \path [-] (B) edge (C);
    \path [-] (D) edge (C);
    \path [-] (E) edge (C);
    \path [-] (D) edge (E);
    \path [-] (D) edge (F);
    \path [-] (E) edge (G);
    \path [-] (F) edge (G);
    \path [-] (F) edge (M);
    \path [-] (H) edge (I);
    \path [-] (H) edge (K);
    \path [-] (K) edge (I);
    \path [-] (I) edge (J);
    \path [-] (I) edge (L);
    \path [-] (L) edge (J);
    \path [-] (K) edge (M);
    \path [-] (K) edge (N);
    \path [-] (K) edge (L);
    \path [-] (M) edge (N);
    \path [-] (O) edge (N);
    \path [-] (O) edge (L);
    \path [-] (P) edge (O);
    \path [-] (a) edge (b);
    \path [-] (a) edge (c);
    \path [-] (c) edge (b);
    \path [-] (d) edge (c);
    \path [-] (e) edge (c);
    \path [-] (d) edge (e);
    \path [-] (d) edge (f);
    \path [-] (e) edge (g);
    \path [-] (f) edge (g);
    \path [-] (f) edge (m);
    \path [-] (h) edge (i);
    \path [-] (h) edge (k);
    \path [-] (k) edge (i);
    \path [-] (i) edge (j);
    \path [-] (i) edge (l);
    \path [-] (l) edge (j);
    \path [-] (k) edge (m);
    \path [-] (k) edge (n);
    \path [-] (k) edge (l);
    \path [-] (m) edge (n);
    \path [-] (o) edge (n);
    \path [-] (o) edge (l);
    \path [-] (p) edge (o);
    \draw[-] (I) .. controls (1.5,1.25) and (1.5,2.75) .. (i);
    
    \path [-] (B1) edge (b1);
    \path [-] (B1) edge (A1);
    \path [-] (C1) edge (A1);
    \path [-] (B1) edge (C1);
    \path [-] (D1) edge (C1);
    \path [-] (E1) edge (C1);
    \path [-] (D1) edge (E1);
    \path [-] (D1) edge (F1);
    \path [-] (E1) edge (G1);
    \path [-] (F1) edge (G1);
    \path [-] (F1) edge (M1);
    \path [-] (H1) edge (I1);
    \path [-] (H1) edge (K1);
    \path [-] (K1) edge (I1);
    \path [-] (I1) edge (J1);
    \path [-] (I1) edge (L1);
    \path [-] (L1) edge (J1);
    \path [-] (K1) edge (M1);
    \path [-] (K1) edge (N1);
    \path [-] (K1) edge (L1);
    \path [-] (M1) edge (N1);
    \path [-] (O1) edge (N1);
    \path [-] (O1) edge (L1);
    \path [-] (P1) edge (O1);
    \path [-] (a1) edge (b1);
    \path [-] (a1) edge (c1);
    \path [-] (c1) edge (b1);
    \path [-] (d1) edge (c1);
    \path [-] (e1) edge (c1);
    \path [-] (d1) edge (e1);
    \path [-] (d1) edge (f1);
    \path [-] (e1) edge (g1);
    \path [-] (f1) edge (g1);
    \path [-] (f1) edge (m1);
    \path [-] (h1) edge (i1);
    \path [-] (h1) edge (k1);
    \path [-] (k1) edge (i1);
    \path [-] (i1) edge (j1);
    \path [-] (i1) edge (l1);
    \path [-] (l1) edge (j1);
    \path [-] (k1) edge (m1);
    \path [-] (k1) edge (n1);
    \path [-] (k1) edge (l1);
    \path [-] (m1) edge (n1);
    \path [-] (o1) edge (n1);
    \path [-] (o1) edge (l1);
    \path [-] (p1) edge (o1);
    \draw[-] (I1) .. controls (1.5,1.25+3.5) and (1.5,2.75+3.5) .. (i1);
    
    \path [-] (B2) edge (b2);
    \path [-] (B2) edge (A2);
    \path [-] (C2) edge (A2);
    \path [-] (B2) edge (C2);
    \path [-] (D2) edge (C2);
    \path [-] (E2) edge (C2);
    \path [-] (D2) edge (E2);
    \path [-] (D2) edge (F2);
    \path [-] (E2) edge (G2);
    \path [-] (F2) edge (G2);
    \path [-] (F2) edge (M2);
    \path [-] (H2) edge (I2);
    \path [-] (H2) edge (K2);
    \path [-] (K2) edge (I2);
    \path [-] (I2) edge (J2);
    \path [-] (I2) edge (L2);
    \path [-] (L2) edge (J2);
    \path [-] (K2) edge (M2);
    \path [-] (K2) edge (N2);
    \path [-] (K2) edge (L2);
    \path [-] (M2) edge (N2);
    \path [-] (O2) edge (N2);
    \path [-] (O2) edge (L2);
    \path [-] (P2) edge (O2);
    \path [-] (a2) edge (b2);
    \path [-] (a2) edge (c2);
    \path [-] (c2) edge (b2);
    \path [-] (d2) edge (c2);
    \path [-] (e2) edge (c2);
    \path [-] (d2) edge (e2);
    \path [-] (d2) edge (f2);
    \path [-] (e2) edge (g2);
    \path [-] (f2) edge (g2);
    \path [-] (f2) edge (m2);
    \path [-] (h2) edge (i2);
    \path [-] (h2) edge (k2);
    \path [-] (k2) edge (i2);
    \path [-] (i2) edge (j2);
    \path [-] (i2) edge (l2);
    \path [-] (l2) edge (j2);
    \path [-] (k2) edge (m2);
    \path [-] (k2) edge (n2);
    \path [-] (k2) edge (l2);
    \path [-] (m2) edge (n2);
    \path [-] (o2) edge (n2);
    \path [-] (o2) edge (l2);
    \path [-] (p2) edge (o2);
    \draw[-] (I2) .. controls (1.5,1.25+7) and (1.5,2.75+7) .. (i2);
    
    \path [-] (A1) edge (a);
    \path [-] (A2) edge (a1);
\end{scope}

\begin{scope}[every node/.style={circle,draw,fill=black,minimum size=1mm,inner sep=1pt},>={Stealth[black]}]
    \node (a) at (4,7.75) {};
    \path [-] (4,7.75) edge (4,8.25);
    \draw[-] (4,7.75) edge (6.5,4.75);
    \node (b) at (4,5.75) {};  
    \path [-] (4,5.75) edge (6.5,10.75); 
    \draw[-] (4,5.75) edge (4,6.25);
    \node (c) at (4,1.25) {};
    \path [-] (4,1.25) edge (6.5,2.75);
    \draw[-] (4,1.25) edge (4,1.75);
    \node (d) at (0+6.5,2.25-1) {};
    \node (e) at (0+6.5,1.25-1) {};
    \node (f) at (0+6.5,0.25-1) {};
    \node (g) at (0.5+6.5,2.5-1) {};
    \node (h) at (0.5+6.5,2-1) {};
    \node (i) at (0.5+6.5,1.5-1) {};
    \node (j) at (0.5+6.5,1-1) {};
    \node (k) at (0.5+6.5,0.5-1) {};
    \node (l) at (0.5+6.5,0-1) {};
    \node (m) at (1+6.5,2.5-1) {};
    \node (n) at (1+6.5,2-1) {};
    \node (o) at (1+6.5,1.5-1) {};
    \node (p) at (1+6.5,1-1) {};
    \node (q) at (1+6.5,0.5-1) {};
    \node (r) at (1+6.5,0-1) {};
    \node (s) at (1.5+6.5,1.5-1) {};
    \node (t) at (2.5+6.5,2.5-1) {};
    \node (u) at (2.5+6.5,0.5-1) {};
    \node (v) at (3.5+6.5,1.5-1) {};
    \path [-] (a) edge (d);
    \path [-] (d) edge (g);
    \path [-] (d) edge (h);
    \path [-] (g) edge (h);
    \path [-] (g) edge (m);
    \path [-] (h) edge (n);
    \path [-] (m) edge (n);
    \path [-] (m) edge (t);
    \path [-] (b) edge (e);
    \path [-] (i) edge (e);
    \path [-] (j) edge (e);
    \path [-] (i) edge (j);
    \path [-] (i) edge (o);
    \path [-] (j) edge (p);
    \path [-] (o) edge (p);
    \path [-] (o) edge (s);
    \path [-] (s) edge (u);
    \path [-] (s) edge (v);
    \path [-] (v) edge (u);
    \path [-] (t) edge (u);
    \path [-] (t) edge (v);
    \path [-] (t) edge (s);
    \path [-] (c) edge (f);
    \path [-] (f) edge (k);
    \path [-] (f) edge (l);
    \path [-] (k) edge (l);
    \path [-] (k) edge (q);
    \path [-] (l) edge (r);
    \path [-] (q) edge (r);
    \path [-] (q) edge (u);
    
    \node (a1) at (4,8.25) {};
    \path [-] (4,8.25) edge (6.5,8.25);
    \draw[-] (4,8.25) edge (4,8.75);
    \node (b1) at (4,4.75) {};   
    \path [-] (4,4.75) edge (6.5,7.25);
    \draw[-] (4,4.75) edge (4,5.25);
    \node (c1) at (4,1.75) {};
    \path [-] (4,1.75) edge (3.5,1.5);
    \node (d1) at (0+6.5,2.25-1+3.5) {};
    \node (e1) at (0+6.5,1.25-1+3.5) {};
    \node (f1) at (0+6.5,0.25-1+3.5) {};
    \node (g1) at (0.5+6.5,2.5-1+3.5) {};
    \node (h1) at (0.5+6.5,2-1+3.5) {};
    \node (i1) at (0.5+6.5,1.5-1+3.5) {};
    \node (j1) at (0.5+6.5,1-1+3.5) {};
    \node (k1) at (0.5+6.5,0.5-1+3.5) {};
    \node (l1) at (0.5+6.5,0-1+3.5) {};
    \node (m1) at (1+6.5,2.5-1+3.5) {};
    \node (n1) at (1+6.5,2-1+3.5) {};
    \node (o1) at (1+6.5,1.5-1+3.5) {};
    \node (p1) at (1+6.5,1-1+3.5) {};
    \node (q1) at (1+6.5,0.5-1+3.5) {};
    \node (r1) at (1+6.5,0-1+3.5) {};
    \node (s1) at (1.5+6.5,1.5-1+3.5) {};
    \node (t1) at (2.5+6.5,2.5-1+3.5) {};
    \node (u1) at (2.5+6.5,0.5-1+3.5) {};
    \node (v1) at (3.5+6.5,1.5-1+3.5) {};
    \path [-] (a1) edge (d1);
    \path [-] (d1) edge (g1);
    \path [-] (d1) edge (h1);
    \path [-] (g1) edge (h1);
    \path [-] (g1) edge (m1);
    \path [-] (h1) edge (n1);
    \path [-] (m1) edge (n1);
    \path [-] (m1) edge (t1);
    \path [-] (b1) edge (e1);
    \path [-] (i1) edge (e1);
    \path [-] (j1) edge (e1);
    \path [-] (i1) edge (j1);
    \path [-] (i1) edge (o1);
    \path [-] (j1) edge (p1);
    \path [-] (o1) edge (p1);
    \path [-] (o1) edge (s1);
    \path [-] (s1) edge (u1);
    \path [-] (s1) edge (v1);
    \path [-] (v1) edge (u1);
    \path [-] (t1) edge (u1);
    \path [-] (t1) edge (v1);
    \path [-] (t1) edge (s1);
    \path [-] (c1) edge (f1);
    \path [-] (f1) edge (k1);
    \path [-] (f1) edge (l1);
    \path [-] (k1) edge (l1);
    \path [-] (k1) edge (q1);
    \path [-] (l1) edge (r1);
    \path [-] (q1) edge (r1);
    \path [-] (q1) edge (u1);
    
    \node (a2) at (4,8.75) {};
    \path [-] (4,8.75) edge (3.5,8.5);
    \node (b2) at (4,5.25) {};   
    \path [-] (4,5.25) edge (3.5,5);
    \node (c2) at (4,2.25) {};
    \path [-] (4,2.25) edge (6.5,9.75);
    \draw[-] (4,2.25) edge (4,2.75);
    \node (d2) at (0+6.5,2.25-1+7) {};
    \node (e2) at (0+6.5,1.25-1+7) {};
    \node (f2) at (0+6.5,0.25-1+7) {};
    \node (g2) at (0.5+6.5,2.5-1+7) {};
    \node (h2) at (0.5+6.5,2-1+7) {};
    \node (i2) at (0.5+6.5,1.5-1+7) {};
    \node (j2) at (0.5+6.5,1-1+7) {};
    \node (k2) at (0.5+6.5,0.5-1+7) {};
    \node (l2) at (0.5+6.5,0-1+7) {};
    \node (m2) at (1+6.5,2.5-1+7) {};
    \node (n2) at (1+6.5,2-1+7) {};
    \node (o2) at (1+6.5,1.5-1+7) {};
    \node (p2) at (1+6.5,1-1+7) {};
    \node (q2) at (1+6.5,0.5-1+7) {};
    \node (r2) at (1+6.5,0-1+7) {};
    \node (s2) at (1.5+6.5,1.5-1+7) {};
    \node (t2) at (2.5+6.5,2.5-1+7) {};
    \node (u2) at (2.5+6.5,0.5-1+7) {};
    \node (v2) at (3.5+6.5,1.5-1+7) {};
    \path [-] (a2) edge (d2);
    \path [-] (d2) edge (g2);
    \path [-] (d2) edge (h2);
    \path [-] (g2) edge (h2);
    \path [-] (g2) edge (m2);
    \path [-] (h2) edge (n2);
    \path [-] (m2) edge (n2);
    \path [-] (m2) edge (t2);
    \path [-] (b2) edge (e2);
    \path [-] (i2) edge (e2);
    \path [-] (j2) edge (e2);
    \path [-] (i2) edge (j2);
    \path [-] (i2) edge (o2);
    \path [-] (j2) edge (p2);
    \path [-] (o2) edge (p2);
    \path [-] (o2) edge (s2);
    \path [-] (s2) edge (u2);
    \path [-] (s2) edge (v2);
    \path [-] (v2) edge (u2);
    \path [-] (t2) edge (u2);
    \path [-] (t2) edge (v2);
    \path [-] (t2) edge (s2);
    \path [-] (c2) edge (f2);
    \path [-] (f2) edge (k2);
    \path [-] (f2) edge (l2);
    \path [-] (k2) edge (l2);
    \path [-] (k2) edge (q2);
    \path [-] (l2) edge (r2);
    \path [-] (q2) edge (r2);
    \path [-] (q2) edge (u2);
    
    \node (a4) at (4,9.5) {};
    \path [-] (4,9.5) edge (3.5,9.5);
    \node (b4) at (4,6.25) {};   
    \path [-] (4,6.25) edge (3.5,6);
    \node (c4) at (4,2.75) {};
    \path [-] (4,2.75) edge (3.5,2.5);
    \node (d4) at (0+6.5,2.25-1+10.5) {};
    \node (e4) at (0+6.5,1.25-1+10.5) {};
    \node (f4) at (0+6.5,0.25-1+10.5) {};
    \node (g4) at (0.5+6.5,2.5-1+10.5) {};
    \node (h4) at (0.5+6.5,2-1+10.5) {};
    \node (i4) at (0.5+6.5,1.5-1+10.5) {};
    \node (j4) at (0.5+6.5,1-1+10.5) {};
    \node (k4) at (0.5+6.5,0.5-1+10.5) {};
    \node (l4) at (0.5+6.5,0-1+10.5) {};
    \node (m4) at (1+6.5,2.5-1+10.5) {};
    \node (n4) at (1+6.5,2-1+10.5) {};
    \node (o4) at (1+6.5,1.5-1+10.5) {};
    \node (p4) at (1+6.5,1-1+10.5) {};
    \node (q4) at (1+6.5,0.5-1+10.5) {};
    \node (r4) at (1+6.5,0-1+10.5) {};
    \node (s4) at (1.5+6.5,1.5-1+10.5) {};
    \node (t4) at (2.5+6.5,2.5-1+10.5) {};
    \node (u4) at (2.5+6.5,0.5-1+10.5) {};
    \node (v4) at (3.5+6.5,1.5-1+10.5) {};
    \path [-] (a4) edge (d4);
    \path [-] (d4) edge (g4);
    \path [-] (d4) edge (h4);
    \path [-] (g4) edge (h4);
    \path [-] (g4) edge (m4);
    \path [-] (h4) edge (n4);
    \path [-] (m4) edge (n4);
    \path [-] (m4) edge (t4);
    \path [-] (b4) edge (e4);
    \path [-] (i4) edge (e4);
    \path [-] (j4) edge (e4);
    \path [-] (i4) edge (j4);
    \path [-] (i4) edge (o4);
    \path [-] (j4) edge (p4);
    \path [-] (o4) edge (p4);
    \path [-] (o4) edge (s4);
    \path [-] (s4) edge (u4);
    \path [-] (s4) edge (v4);
    \path [-] (v4) edge (u4);
    \path [-] (t4) edge (u4);
    \path [-] (t4) edge (v4);
    \path [-] (t4) edge (s4);
    \path [-] (c4) edge (f4);
    \path [-] (f4) edge (k4);
    \path [-] (f4) edge (l4);
    \path [-] (k4) edge (l4);
    \path [-] (k4) edge (q4);
    \path [-] (l4) edge (r4);
    \path [-] (q4) edge (r4);
    \path [-] (q4) edge (u4);
\end{scope}

\begin{scope}[every node/.style={circle,draw,fill=black,minimum size=1mm,inner sep=1pt},>={Stealth[black]}]
    \node (a) at (0,11.25) {};
    \node (b) at (0.5,11.5) {};
    \node (c) at (0.5,11) {};
    \node (d) at (1,11.5) {};
    \node (e) at (1,11) {};
    \path [-] (a) edge (b);
    \path [-] (a) edge (c);
    \path [-] (c) edge (b);
    \path [-] (d) edge (b);
    \path [-] (c) edge (e);
    \path [-] (d) edge (e);
    \path [-] (a) edge (0,10.5);
    
    \node (f) at (0,-1.5) {};
    \path[-] (0,0.5) edge (f);
    
    \node (g) at (11,-1.5) {};
    \node (h) at (11,0.5) {};
    \node (i) at (11,4) {};
    \node (j) at (11,7.5) {};
    \node (k) at (11,11) {};                
    \path [-] (f) edge (g);
    \path [-] (g) edge (h);
    \path [-] (h) edge (i);
    \path [-] (i) edge (j);
    \path [-] (j) edge (k);
    \path [-] (h) edge (10,0.5);
    \path [-] (i) edge (10,4);
    \path [-] (j) edge (10,7.5);
    \path [-] (k) edge (10,11);
\end{scope}
\end{tikzpicture}
\caption{The reduction from a satisfiable 3-CNF $(X \vee Y \vee Z)\wedge(\bar{X} \vee \bar{Y} \vee Z)\wedge(\bar{X} \vee \bar{Y} \vee \bar{Z})\wedge(\bar{X} \vee Y \vee \bar{Z})$ to a moral graph with maximum degree $5$. From top to bottom, the variable gadgets are for $X, Y, Z$ and the clause gadgets are for $F_1,F_2,F_3,F_4$.}
\label{fg:3cnf_fully_dir}
\end{figure}
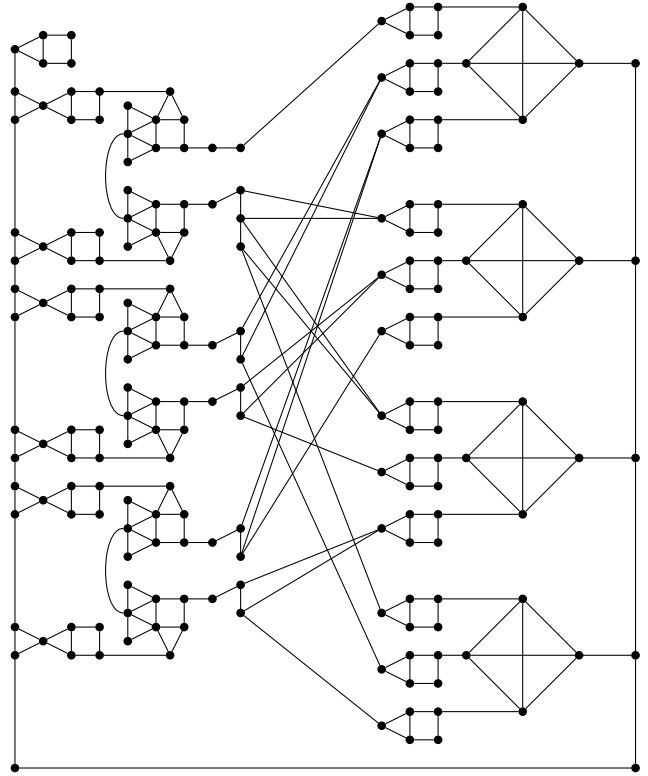

\newpage 
\bibliography{causal_discovery_ref_list}
\bibliographystyle{icml2019}
\end{document}